\documentclass{article}

\usepackage[preprint]{neurips_2021}

\usepackage[utf8]{inputenc} % allow utf-8 input
\usepackage[T1]{fontenc}    % use 8-bit T1 fonts
\usepackage{hyperref}       % hyperlinks
\usepackage{url}            % simple URL typesetting
\usepackage{booktabs}       % professional-quality tables
\usepackage{amsfonts}       % blackboard math symbols
\usepackage{nicefrac}       % compact symbols for 1/2, etc.
\usepackage{microtype}      % microtypography
\usepackage{xcolor}         % colors
\usepackage{amsmath}
\usepackage{amssymb}
\usepackage{algorithm}
\usepackage{algorithmic}
\usepackage{booktabs}
\usepackage{siunitx}
\usepackage{amsthm}
\newtheorem{lemma}{Lemma}
\usepackage{graphicx}

\title{Non Gaussian Denoising Diffusion Models}

\author{
Eliya Nachmani* \\
Tel-Aviv University \\
Facebook AI Research \\
\texttt{enk100@gmail.com} \\
\And 
Robin San Roman* \\
École Normale Supérieure Paris-Saclay \\
\texttt{sanroman.robin@gmail.com} \\  
\And
Lior Wolf\\
Tel-Aviv University\\
\texttt{wolf@cs.tau.ac.il} \\  
}

\begin{document}

\maketitle
{\let\thefootnote\relax\footnote{{*Equal contribution}}}
\begin{abstract}
  Generative diffusion processes are an emerging and effective tool for image and speech generation. In the existing methods, the underline noise distribution of the diffusion process is Gaussian noise. However, fitting distributions with more degrees of freedom, could help the performance of such generative models. In this work, we investigate other types of noise distribution for the diffusion process. Specifically, we show that noise from Gamma distribution provides improved results for image and speech generation. Moreover, we show that using a mixture of Gaussian noise variables in the diffusion process improves the performance over a diffusion process that is based on a single distribution. Our approach preserves the ability to efficiently sample state in the training diffusion process while using Gamma noise and a mixture of noise. 
\end{abstract}

\section{Introduction}
Deep generative neural networks has shown significant progress over the last years. The main architectures for generation are: (i) VAE \cite{kingma2013auto} based, for example, NVAE \cite{vahdat2020nvae} and VQ-VAE \cite{razavi2019generating}, (ii) GAN \cite{goodfellow2014generative} based, for example, StyleGAN \cite{karras2020analyzing} for vision application and WaveGAN \cite{donahue2018adversarial} for speech
(iii) Flow-based, for example Glow \cite{kingma2018glow} (iv) Autoregessive, for example, Wavenet for speech \cite{oord2016wavenet} and (v) Diffusion Probabilistic Models \cite{sohl2015deep}, for example, Denoising Diffusion Probabilistic Models (DDPM) \cite{ho2020denoising} and its implicit version DDIM \cite{song_denoising_2020}.

Models from this last family have shown significant progress in generation capabilities during the last years, e.g., \cite{sohl2015deep, ho2020denoising, chen_wavegrad_2020,kong_diffwave_2020}, and have achieved comparable results to the state of the art generation architecture in both images and speech. % such as StyleGAN2 \cite{karras2020analyzing} and Wavenet \cite{oord2016wavenet}. 

A DDPM is a Markov chain of latent variables. Two processes are modeled: (i) a diffusion process and (ii) a denoising process. During training, the diffusion process learns to transform data samples into Gaussian noise. The denoising is the reverse process and it is used during inference to generate data samples, stating from Gaussian noise. The second process can be condition on attributes to control the generation sample. In order to get high quality synthesis, a large number of denoising steps is used (i.e. $1000$ steps). A notable property of the diffusion process is a closed form formulation of the noise that arises from accumulating diffusion stems. This allows to sampling arbitrary states in the Markov chain of the diffusion process without calculating the previous steps. 

In the Gaussian case, this property stems from the fact that adding Gaussian distributions leads to another Gaussian distribution. Other distributions have similar properties. For example, for the Gamma distribution, the sum of two distributions that share the scale parameter is a Gamma distribution of the same scale. The Poisson distribution has a similar property. However, its discrete nature makes it less suitable for DDPM. 

In DDPM, the mean of the distribution is set at zero. The Gamma distribution, with its two parameters (shape and scale), is better suited to fit the data than a Gaussian distribution with one degree of freedom (scale). Furthermore, the Gamma distribution generalizes other distributions, and many other distributions can be derived from it~\cite{leemis2008univariate}.

The added modeling capacity of the Gamma distribution can help speed up the convergence of the DDPM model. Consider, for example, conventional DDPM model that were trained with Gaussian noise on the CelebA dataset~\cite{liu2015faceattributes}. After $t$ steps, one can compute the histogram of the differences between the obtained image $x_t$ and the original noise image $x_0$. Both a Gaussian distribution, Gamma distribution, and a mixture of Gaussian can then be fitted to this histogram, as shown in Fig.~\ref{fig:fitting_error}(a,b,c). Unsurprisingly, for a single step, both the Gaussian distribution and the Gamma distribution can be fitted with the same error, see Fig.~\ref{fig:fitting_error}(d). However, when trying to fit for more than a single step, the Gamma distribution and the mixture of Gaussian becomes a much better fit.

In this paper, we investigate two types of non-Gaussian noise distribution: (i) Mixture of Gaussian, and (ii) Gamma. The two proposed models maintain the property of the diffusion process to sample arbitrary states without calculating the previous steps. Our results are demonstrated in two major domains: vision and audio. In the first domain, the proposed method is shown to provide a better FID score for generated images. For speech data, we show that the proposed method improves various measures, such as Perceptual Evaluation of Speech Quality (PESQ), short-time objective intelligibility (STOI) and Mel-Cepstral Distortion (MCD).

\begin{figure}[]
    \centering
    \begin{tabular}{cccc}
    \includegraphics[width=.2125\textwidth,keepaspectratio]{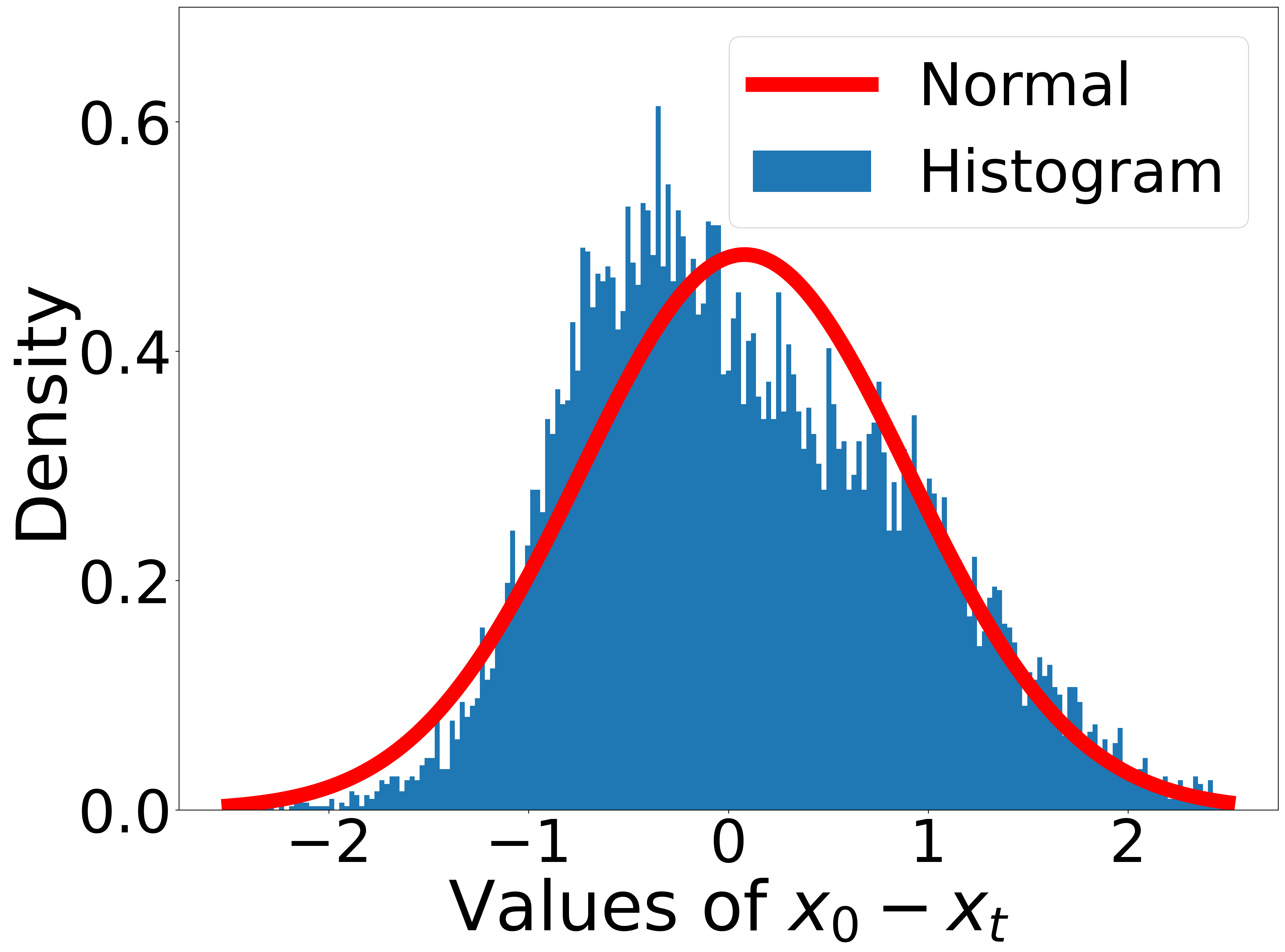} &
    \includegraphics[width=.2125\textwidth,keepaspectratio]{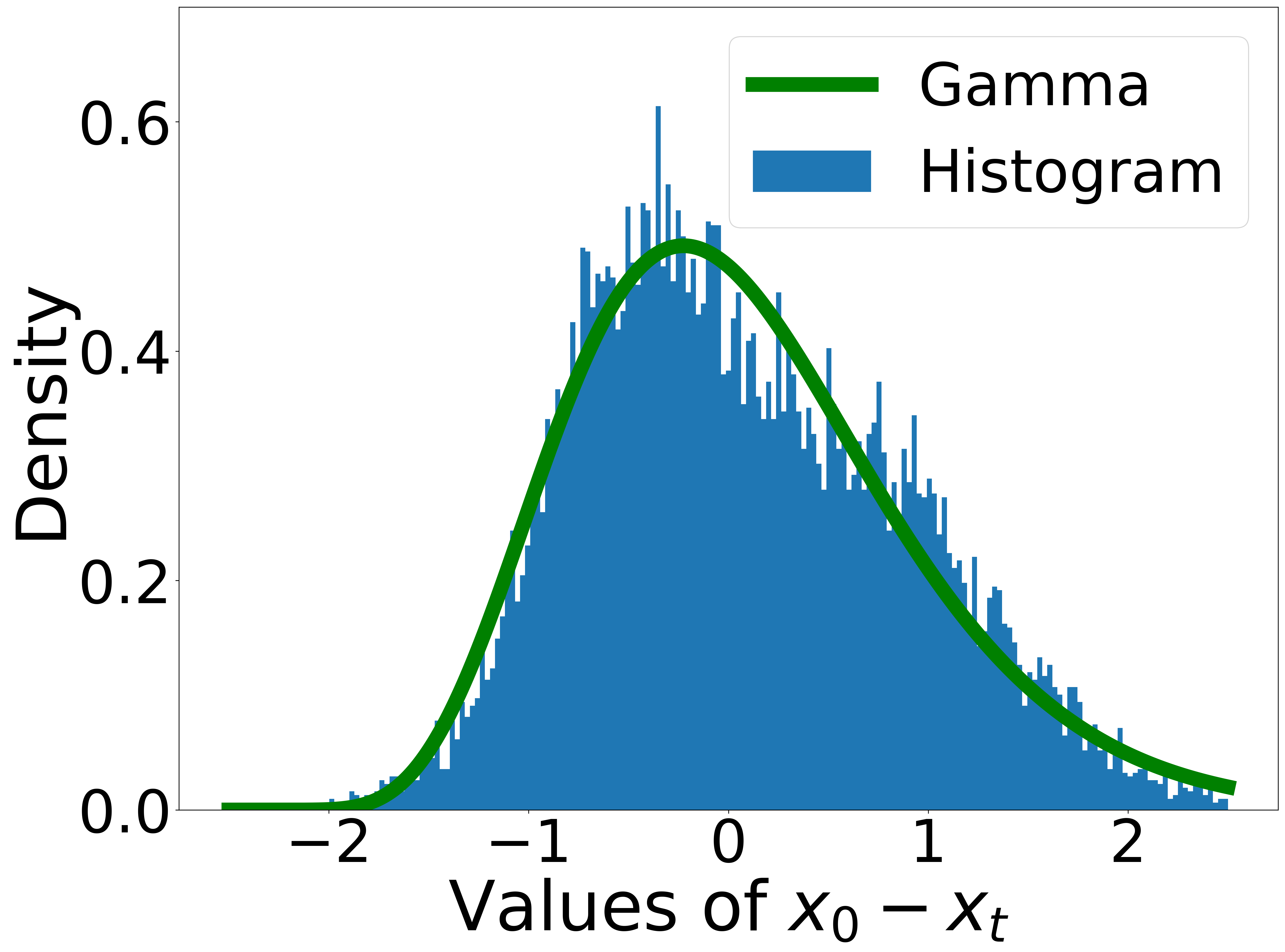} &
    \includegraphics[width=.2125\textwidth,keepaspectratio]{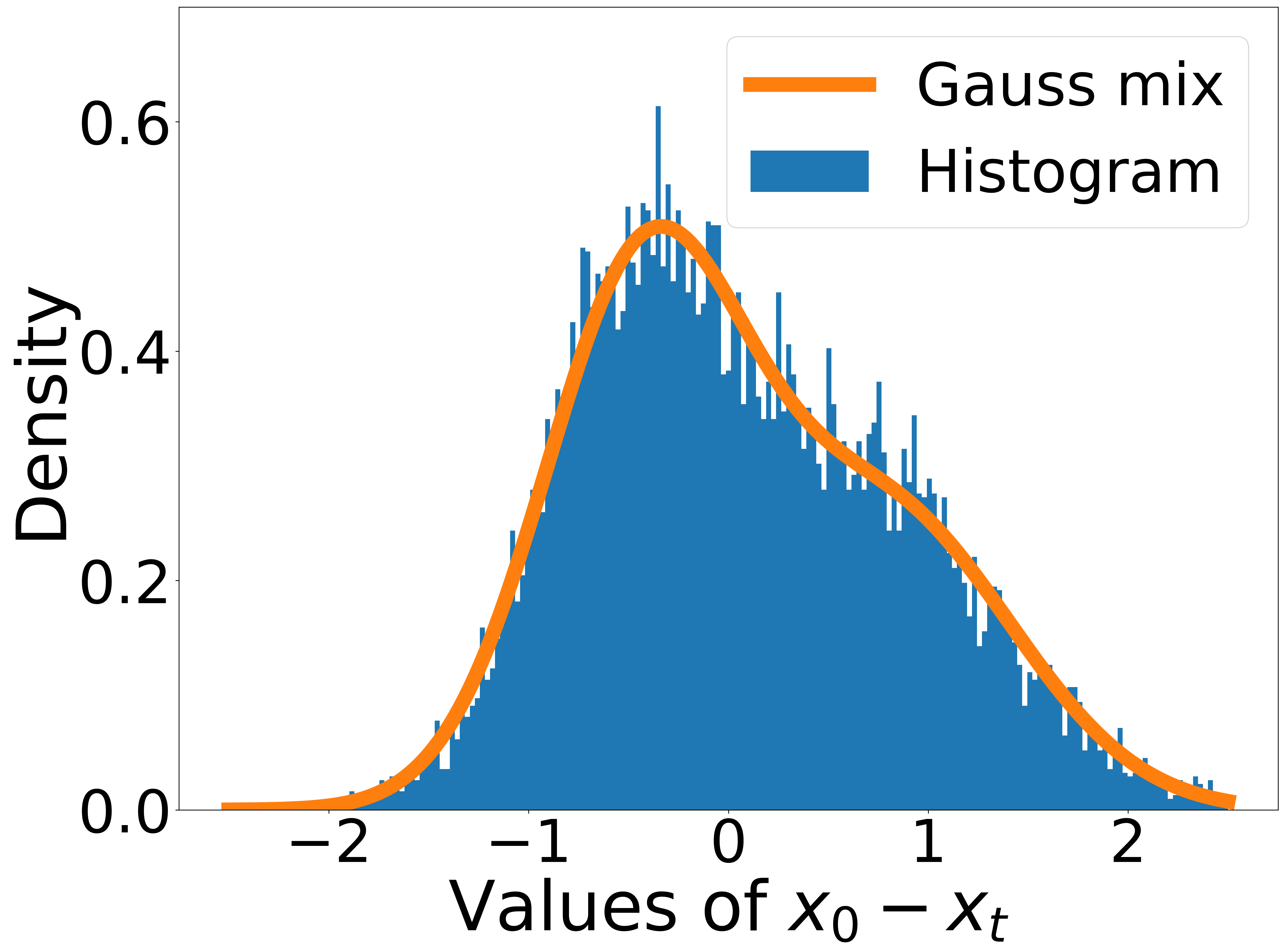} &
    \includegraphics[width=.2125\textwidth,keepaspectratio]{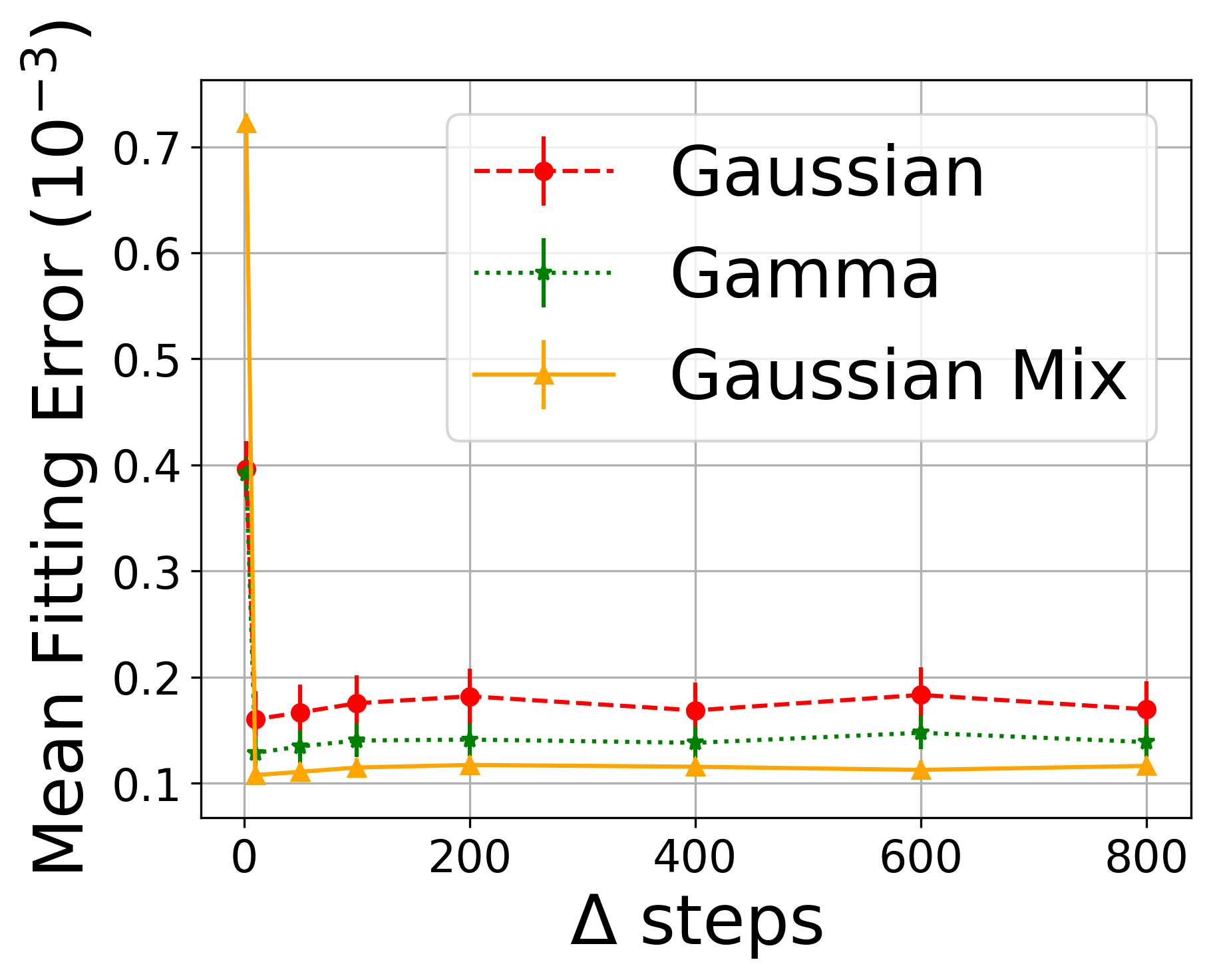} \\
    (a) & (b) & (c) & (d) \\
    \end{tabular}
    \caption{Fitting a distribution to the histogram of the difference between $x_0$ and the image $x_t$ after $t$ DDPM steps, using a pretrained (Gaussian) celebA (64x64) model. (a) The fitting of a Gaussian to the histogram of a typical image after $t-50$ steps. (b) Fitting a mixture of Gaussian distribution. (c) Fitting a Gamma distribution. (d) The fitting error to Gaussian and Gamma distribution, measured as the MSE between the histogram and the fitted probability distribution function. Each point is the average value for the generation of $100$ images. The vertical errorbars denote the standard deviation.}
    \label{fig:fitting_error}
\end{figure}

\section{Related Work} 
In their seminal work, Sohl-Dickstein et al. introduce the Diffusion Probabilistic Model~\cite{sohl2015deep}. The model is applied to various domains, such as time series and images. The main drawback in the proposed model is the fact that it needs up to thousands of iterative steps in order to generate valid data sample. Song et al.~\cite{song2019generative} proposed a diffusion generative model based on Langevin dynamics and the score matching method~\cite{hyvarinen2005estimation}. The model estimate the Stein score function~\cite{liu2016kernelized} which is the logarithm of the data density. Given the Stein score function, the model can generate data points. 

Denoising Diffusion Probabilistic Models (DDPM)~\cite{ho2020denoising} combine generative models based on score matching and neural Diffusion Probabilistic Models into a single model. Similarly, in \cite{chen_wavegrad_2020,kong2020diffwave} a generative neural diffusion process based on score matching was applied to speech generation. These models achieve state of the art results for speech generation, and show superior results over well-established methods, such as Wavernn \cite{kalchbrenner2018efficient}, Wavenet \cite{oord2016wavenet}, and GAN-TTS \cite{binkowski2019high}.

Diffusion Implicit Models (DDIM) is a method to accelerate the denoising process~\cite{song_denoising_2020}. The model employs a non-Markovian diffusion process to generate higher quality sample. The model helps to reduce the number of diffusion steps, e.g., from a thousand steps to a few hundred.

Song et al. show that score based generative models can be considered as a solution to a stochastic differential equation~\cite{song2020score}. Gao et al. provide an alternative approach to train an energy-based generative model using a diffusion process~\cite{gao2020learning}.

Another line of work in audio is that of neural vocoders that are based on a denoising diffusion process. WaveGrad~\cite{chen_wavegrad_2020} and DiffWave~\cite{kong2020diffwave} are conditioned on the mel-spectrogram and produce high fidelity audio samples using as few as six step of the diffusion process. These models outperform adversarial non-autoregressive baselines. 

\begin{figure}[!t]
\begin{minipage}[t]{0.49\textwidth}
\begin{algorithm}[H]
  \caption{DDPM training procedure.} 
  \label{alg:DDPM_train}
  \begin{algorithmic}[1]
    \STATE Input: dataset $d$, diffusion process length $T$, noise schedule $\beta_1,...,\beta_T$
   \REPEAT
   \STATE $ x_0 \sim d(x_0)$
   \STATE $t \sim \mathcal{U}(\{1, ..., T\})$
   \STATE $\varepsilon \sim \mathcal{N}(0, I)$
   \STATE $x_t = \sqrt{\bar \alpha_t}x_0 + \sqrt{1 - | \bar \alpha_t|}\varepsilon$
   \STATE Take gradient descent step on: \newline $\| \varepsilon - \varepsilon_{\theta}(x_t, t )\|_1 $
   \UNTIL{converged}
  \end{algorithmic}
\end{algorithm}
\end{minipage}
\hfill
\begin{minipage}[t]{0.49\textwidth}
\begin{algorithm}[H]
  \caption{DDPM sampling algorithm}%
  \label{alg:DDPM_inf}
  \begin{algorithmic}[1]
  \STATE $x_T \sim \mathcal{N}(0, I) $
   \FOR{t= T, ..., 1}
   \STATE $ z \sim \mathcal{N}(0, I)$
   \STATE $\hat \varepsilon = \varepsilon_\theta(x_t, t)$
   \STATE $x_{t-1} = \frac{x_t - \frac{1 - \alpha_t}{\sqrt{1 - \bar \alpha_t}}\hat \varepsilon }{\sqrt{\alpha_t}}$ 
   \IF{$t \neq 1$}
   \STATE $x_{t-1} = x_{t-1} + \sigma_t z$ 
   \ENDIF
   \ENDFOR
   \STATE \textbf{return} $x_0$
  \end{algorithmic}
\end{algorithm}
\end{minipage}
\end{figure}

\section{Diffusion models for Various Distributions}\label{section:method} 

We start by recapitulating the Gaussian case, after which we derive diffusion models for mixtures of Gaussians and for the Gamma distribution.

\subsection{Background - Gaussian DDPM} 
Diffusion network learn the gradients of the data log density
\begin{equation}
    s(y) = \nabla_y \log p(y) 
\end{equation}
By using Langevin Dynamics and the gradients of the data log density $\nabla_y \log p(y)$, a sample procedure from the probability can be done by:
\begin{equation}\label{eq:langevin}
    \tilde y_{i+1} = \tilde y_i + \frac{\eta}{2} s(\tilde y_i) + \sqrt{\eta}z_i
\end{equation}
where $z_i\sim \mathcal{N}(0, I)$ and $\eta > 0$ is the step size. 
The diffusion process in DDPM \cite{ho2020denoising} is defined by a Markov chain that gradually adds Gaussian noise to the data according to a noise schedule. The diffusion process is defined by:
\begin{equation}
    q(x_{1:T}|x_0) = \prod_{t=1}^{T} q(x_t|x_{t-1})\,,
\end{equation}
where T is the length of the diffusion process, and $x_T,...,x_t,x_{t-1},...,x_0$ is a sequence of latent variables with the same size as the clean sample $x_0$.
The Diffusion process is parameterized with a set of parameters called noise schedule ($\beta_1, \dots \beta_T$) that defines the variance of the noise added at each step:
\begin{equation}
\label{eq:ddpm_diff_normal}
     q(x_t|x_{t-1}) := \mathcal{N}(x_{t}; \sqrt{1 - \beta_t}x_{t-1}, \beta_t\mathbf{I})\,,
\end{equation}

Since we are using Gaussian noise random variable at each step The diffusion process can be simulated for any number of steps with the closed formula: 
\begin{equation}\label{eq:closedy_n}
    x_t = \sqrt{\bar \alpha_t} x_0 + \sqrt{1 - \bar\alpha_t} \varepsilon\,,
\end{equation}

where $\alpha_i = 1 - \beta_i$, $\bar \alpha_t = \prod_{i=1}^t \alpha_i$ and $\varepsilon = \mathcal{N}(0,\mathbf{I})$. The complete training procedure is defined in Alg.\ref{alg:DDPM_train}. Given the input dataset $d$, the algorithm samples $\epsilon$, $x_0$ and $t$. The noisy latent state $x_t$ is calculated and fed to the DDPM neural network $\varepsilon_\theta$. A gradient descent step is taken in order to estimate the $\varepsilon$ noise with the DDPM network $\varepsilon_\theta$. 

The inference procedure uses the trained model $\varepsilon_{\theta}$ and a variation of the Langevin dynamics. The following update from \cite{song_denoising_2020} is used to reverse a step of the diffusion process:
\begin{equation}\label{eq:DDPM_update}
x_{t-1} = \dfrac{x_t - \frac{1 - \alpha_t}{\sqrt{1- \bar\alpha_t}}\varepsilon_\theta(x_t, t)}{\sqrt{\bar\alpha_t}} + \sigma_t \varepsilon\,,
\end{equation}
where $\varepsilon$ is white noise and $\sigma_t$ is the standard deviation of added noise. In \cite{song_denoising_2020} the authors use ${\sigma_t}^2 = \beta_t$. The complete inference algorithm present at Alg.~\ref{alg:DDPM_inf}. Starting from a Gaussian noise and then step-by-step reversing the diffusion process, by iteratively employing the update rule of Eq.~\ref{eq:DDPM_update}.

%%%%
%%%%

\subsection{Mixture of Gaussian Noise} 
Eq.~\ref{eq:ddpm_diff_normal} can be written as:
\begin{equation}\label{eq:difusion_step}
	 x_t = \sqrt{1 - \beta_t} x_{t-1} + \sqrt{\beta_t}\epsilon_t
\end{equation}
where $\epsilon_t$ is the Gaussian noise of step $t$. This can be generalized by adding a mixture of Gaussians at each step, i.e., 
\begin{equation}\label{eq:mix_gauss}
	 x_t = \sqrt{1 - \beta_t} x_{t-1} + \sqrt{\beta_t}(\sum_{i=0}^{C}{z_i\epsilon^i_t})
\end{equation}
where $z_i$ are boolean random variables and $C$ is the number of Gaussian variables. We denote by $p_i$ the probability that $z_i=1$ for the $i$ Gaussian variable ($\sum_{i=0}^{C} p_i =1$). 

For simplicity, we focus on $C=2$, mixtures with an expectation of zero, and two Gaussian distributions with the same variance $\phi_t^2$ and the same weight ($p=0.5$):
\begin{equation}\label{eq:idea1}
	 x_t = \sqrt{1 - \beta_t} x_{t-1} + \sqrt{\beta_t}(b \epsilon^1_t + (1-b)\epsilon^2_t)\,,
\end{equation}
where $\epsilon^1_t \sim N(m^1_t, {\phi_t}^2)$, $\epsilon^1_t \sim N(m^2_t, {\phi_t}^2)$ and $b \sim \mathcal Bernoulli(p)$. In addition to the zero mean property that we assume, we scale the variables such that the variance of $b \epsilon^1_t + (1-b)\epsilon^2_t$ is one.

We reparametrize the added noise at each step as $\sqrt{\beta_t}X_t$, i.e.
\begin{equation}
    X_t = b \epsilon^1_t + (1-b)\epsilon^2_t\,.
\end{equation}

Since $X_t$ is zero mean and has a variance of one ($E(X_t)=0$ and $V(X_t)=1$), the added value at each step $\sqrt{\beta_t}X_t$ will have zero mean and a $\beta_t$ variance. The following Lemma captures the relation between the two means in this case.

\begin{lemma}
Assuming that $\epsilon_t^1 \sim \mathcal{N}(m_t^1, {\phi_t}^2)$, $\epsilon_t^2 \sim \mathcal{N}(m_t^2, {\phi_t}^2)$, $b\sim Bernouilli(p)$, $E(X_t) = 0$ and $V(X_t) = 1$. Then the following equation hold:
\begin{equation}\label{eq:mean_relation}
	m_t^1 = \sqrt{\frac{1-{\phi_t}^2}{p(1-p) + \frac{p^3}{1-p} + 2p^2}}
\end{equation}
\begin{equation}\label{eq:mean_relation2}
	m_t^2 = -\frac{p}{1-p}m_t^1
\end{equation}
\end{lemma}
\begin{proof}
     The mean value of $X_t$ can be calculate as:
    \begin{align*}
       E(X_t) &= E(b \epsilon^1_t) + E((1-b)\epsilon^2_t)  = E(b)E(\epsilon^1_t) + E(1-b)E(\epsilon^2_t) \\ &= E(b)E(\epsilon^1_t) + E(1-b)E(\epsilon^2_t)  = pm_t^1 + (1-p)m_t^2
    \end{align*}
    
    since $b, \epsilon^1_t, \epsilon^2_t$ are independent and due to the linearity of the expectation.
    Since we assume that $E(X_t)=0$, we have:
    \begin{equation}\label{eq:mean_relation_gmm_p}
    	m_t^2 = -\frac{p}{1-p}m_t^1
    \end{equation}
    
    Denote $X^1_t= b\epsilon^1_t$ and $X^2_t =(1-b)\epsilon^2_t$. The variance of $X_t$ is given as:
    \begin{equation}
    \label{eq:var_base}
        V(X_t) = V(X^1_t) + V(X^2_t) + 2cov(X^1_t, X^2_t)
    \end{equation}
    where $cov(X^1_t, X^2_t) = - E(X^1_t)E(X^2_t)$
    since $X^1_tX^2_t = 0$. Therefore, Eq.\ref{eq:var_base} becomes:
    \begin{equation}
    \label{eq:var_all}
        V(X_t) = V(X^1_t) + V(X^2_t) - 2E(X^1_t)E(X^2_t)
    \end{equation}
    
    $V(X^1_t)$ can be calculate as:
    \begin{align*}
        V(X^1_t) &= V(b\epsilon^1_t) = V(\epsilon^1_t)E(b^2) + V(b)E(\epsilon^1_t)^2 = 
     {\phi_t}^2E(b) + p(1-p)({m_t^1})^2 \nonumber \\ &= {\phi_t}^2p + p(1-p)({m_t^1})^2
    \end{align*}
    since $b^2 = b$. Similarly, we can show that $V(X^2_t) = {\phi_t}^2(1-p) + p(1-p)({m_t^2})^2$. 
    Therefore, we have:
    \begin{align}
        V(X_t) &= {\phi_t}^2p + p(1-p)({m_t^1})^2 + {\phi_t}^2(1-p) + p(1-p)({m_t^2})^2 - 2E(X^1_t)E(X^2_t) \nonumber
        \\ &= {\phi_t}^2p + p(1-p)({m_t^1})^2 + {\phi_t}^2(1-p) + p(1-p)({m_t^2})^2 - 2pm_t^1(1-p)m_t^2 \label{mid_var_gmm_p}
    \end{align}
    
    Substitute Eq.\ref{eq:mean_relation_gmm_p} into Eq.\ref{mid_var_gmm_p} we have:
    $V(X_t) = ({m_t^1})^2 \left( p(1-p) + \frac{p^3}{1-p} + 2p^2 \right) +p {\phi_t}^2 + (1-p){\phi_t}^2$. Since $V(X_t) = 1$ we have:
    \begin{equation}\label{eq:mean_var_relation}
    	m_t^1 = \sqrt{\frac{1-{\phi_t}^2}{p(1-p) + \frac{p^3}{1-p} + 2p^2}}
    \end{equation}

\end{proof}

Denote $\mathcal{M}({\phi_t}^2)$ ($\phi_t \in [0,1]$) as the mixture model with two equally weighed Gaussian components, each with a variance of $\phi_t^2$ and means stated above. The closed form for sampling $x_t$ from $x_0$ is given by:
\begin{equation}
\label{eq:closed_form_gauss_mix}
    x_t = \sqrt{\bar \alpha_t} x_0 + \sqrt{1 - \bar\alpha_t}N_t
\end{equation}
Where $N_t\sim \mathcal{M}({\phi_t}^2)$ and $\phi_t$ is a free hyperparameter.
Similarly to Eq.\ref{eq:DDPM_update}, the inference is given by:
\begin{equation}\label{eq:sample_gauss_mix}
x_{t-1} = \dfrac{x_t - \frac{1 - \alpha_t}{\sqrt{1- \bar\alpha_t}}\varepsilon_\theta(x_t, t)}{\sqrt{\bar\alpha_t}} + \sigma_t N_t
\end{equation}

The complete training procedure is given in Alg.~\ref{alg:mix_gauss_train}. The input is the standard deviation for each step ($\phi_t)_{t \in \{1... T\} }$, the dataset $d$, the total number of steps in the diffusion process $T$ and the noise schedule $\beta_1,...,\beta_T$. For each batch, the training algorithm samples an example $x_0$, a number of steps $t$ and the noise itself $\varepsilon$. Then it calculates $x_t$ from $x_0$ by using Eq.\ref{eq:closed_form_gauss_mix}. The neural network $\varepsilon_{\theta}$ has an input $x_t$ and is condition on the number of step $t$.
A gradient descend step is used to approximate $N_t$ with the neural network $\varepsilon_{\theta}$. The main differences from the single Gaussian case (Alg.~\ref{alg:DDPM_train}) are the following: (i) calculating the mixture of Gaussian parameters (ii) and sampling $N_t$.

The inference procedure is given in Alg.~\ref{alg:mix_gauss_infernce}. The algorithm starts from a noise $x_T$ sampled from $\mathcal{M}({\phi_T}^2)$. Then for $T$ steps the algorithm estimates $x_{t-1}$ from $x_t$ by using Eq.\ref{eq:sample_gauss_mix}. Note that as in \cite{song_denoising_2020} $\sigma_t=\beta_t$. The main differences from the single Gaussian case (Alg.~\ref{alg:DDPM_inf}) is the following: (i) the start sampling point $x_T$, (ii) and the sampling noise $z$.

\begin{figure}[!t]
\begin{minipage}[t]{0.49\textwidth}
\begin{algorithm}[H]
  \caption{Mixture of Gaussian Training Algorithm} 
  \label{alg:mix_gauss_train}
  \begin{algorithmic}[1]
   \STATE Input: The standard deviations ($\phi_t)_{t \in \{1... T\} }$, dataset $d$, diffusion process length $T$, noise schedule $\beta_1,...,\beta_T$
   \REPEAT
   \STATE $ x_0 \sim d(x_0)$
   \STATE $t \sim \mathcal{U}(\{1, ..., T\})$
   \STATE $N_t \sim \mathcal{M}({\phi_t}^2)$
   \STATE $x_t = \sqrt{\bar \alpha_t}x_0 + \sqrt{1 - | \bar \alpha_t|}N_t$
   \STATE Take gradient descent step on: \newline $\| N_t - \varepsilon_{\theta}(x_t, t)\|_1 $
   \UNTIL{converged}
  \end{algorithmic}
\end{algorithm}
\end{minipage}
\hfill
\begin{minipage}[t]{0.49\textwidth}
\begin{algorithm}[H]
  \caption{Mixture of Gaussian Inference Algorithm}%
  \label{alg:mix_gauss_infernce}
  \begin{algorithmic}[1]
  \STATE $x_T \sim \mathcal{M}({\phi_T}^2) $
   \FOR{t= T, ..., 1}
   \STATE $ z \sim \mathcal{M}({\phi_{t-1}}^2)$
   \STATE $\hat \varepsilon = \varepsilon_\theta(x_t, t)$
   \STATE $x_{t-1} = \frac{x_t - \frac{1 - \alpha_t}{\sqrt{1 - \bar \alpha_t}}\hat \varepsilon }{\sqrt{\alpha_t}}$ 
   \IF{$t \neq 1$}
   \STATE $x_{t-1} = x_{t-1} + \sigma_t z$ 
   \ENDIF
   \ENDFOR
   \STATE \textbf{return} $x_0$
  \end{algorithmic}
\end{algorithm}
\end{minipage}
\end{figure}

\subsection{Using the Gamma distribution for noise} 

Denote $\Gamma(k, \theta)$ as the Gamma distribution, where $k$ and $\theta$ are the shape and the scale respectively. We modify Eq.~\ref{eq:difusion_step} by adding, during the diffusion process, noise that follows a Gamma distribution:
\begin{equation}\label{eq:gamma_t}
	 x_t = \sqrt{1 - \beta_t} x_{t-1} + (g_t - \mathbb{E}(g_t))
\end{equation}
where $g_t\sim \Gamma(k_t, \theta_t)$, $\theta_t = \sqrt{\bar \alpha_t}\theta_0$ and $k_t=\dfrac{\beta_t}{\alpha_t{\theta_0}^2}$. Note that $\theta_0$ and $\beta_t$ are hyperparameters.

Since the sum of Gamma distribution (with the same scale parameter) is distributed as Gamma distribution, one can derive a closed form for $x_t$, i.e. an equation to calculate $x_t$ from $x_0$:
\begin{equation}
\label{eq:close_form_single_gamma}
 x_t = \sqrt{\bar \alpha_t} x_0 + (\bar g_t - \bar k_t\theta_t)
\end{equation}
where $\bar g_t \sim \Gamma(\bar k_t, \theta_t)$ and $\bar k_t = \sum_{i=1}^t k_i$.

\begin{lemma}
    Let $\theta_0 \in \mathbb{R}$, 
    Assuming $\forall t \in \{1,..., T\}$, $k_t=\dfrac{\beta_t}{\alpha_t{\theta_0}^2}$, 
    $\theta_t = \sqrt{\bar \alpha_t}\theta_0$, and $g_t\sim \Gamma(k_t, \theta_t)$. Then $\forall t \in \{1,..., T\}$ the following hold:
    
    \begin{equation}\label{eq:lemma2_1}
        E(g_t - E(g_t)) =0, V(g_t - E(g_t)) =\beta_t
    \end{equation}
    
    \begin{equation}\label{eq:lemma2_2}
        x_t = \sqrt{\bar \alpha_t}x_0 + (\bar g_t - E(\bar g_t))
    \end{equation}
    where $\bar g_t \sim \Gamma(\bar k_t, \theta_t)$ and $\bar k_t = \sum_{i=1}^t k_i$
\end{lemma}
\begin{proof}
    
    The first part of Eq.~\ref{eq:lemma2_1} is immediate. The variance part is also straightforward:
    \begin{equation*}
        V(g_t - E(g_t))= k_t{\theta_t}^2 = \beta_t
    \end{equation*}

    Eq.~\ref{eq:lemma2_2} is proved by induction on $t\in \{1, ...T\}$. 
    For $t=1$:
    \begin{equation*}
        x_1 = \sqrt{1-\beta_1}x_0 + g_1 - E(g_1)
    \end{equation*}

    since $\bar k_1 = k_1$, $\bar g_1 = g_1$. We also have that $\sqrt{1-\beta_1} = \sqrt{\bar \alpha_1}$. Thus we have:

\begin{equation*}
    x_1 = \sqrt{\bar \alpha_1}x_0 + (\bar g_1 - E(\bar g_1))
\end{equation*}

Assume Eq.~\ref{eq:lemma2_2} holds for some $t\in \{1, ...T\}$. The next iteration is obtained as
\begin{align}
    x_{t+1} &= \sqrt{1- \beta_{t+1}} x_t + g_{t+1} - E(g_{t+1})\\
           & = \sqrt{1- \beta_{t+1}} (\sqrt{\bar \alpha_t}x_0 + (\bar g_t - E(\bar g_t))) + g_{t+1} - E(g_{t+1}) \\
           & = \sqrt{\bar \alpha_{t+1}}x_0 + \sqrt{1- \beta_{t+1}} \bar g_t + g_{t+1} - (\sqrt{1- \beta_{t+1}}E(\bar g_t) + E(g_{t+1}))
\end{align}

It remains to be proven that (i) $\sqrt{1- \beta_{t+1}} \bar g_t + g_{t+1} = \bar g_{t+1}$ and (ii) $\sqrt{1- \beta_{t+1}}E(\bar g_t) + E(g_{t+1}) = E(\bar g_{t+1})$. Since $\bar g_t \sim \Gamma(\bar k_t, \theta_t)$ hold, then:
\begin{align*}
    \sqrt{1- \beta_{t+1}} \bar g_t & \sim \Gamma(\bar k_t, \sqrt{1- \beta_{t+1}}\theta_t) = \Gamma(\bar k_t, \theta_{t+1}) 
\end{align*}
Therefore, we prove (i):
\begin{align*}
    \sqrt{1- \beta_{t+1}} \bar g_t + g_{t+1} &\sim \Gamma(\bar k_t+ k_{t+1}, \theta_{t+1}) = \Gamma(\bar k_{t+1}, \theta_{t+1})
\end{align*}
which implies that $\sqrt{1- \beta_{t+1}} \bar g_t + g_{t+1}$ and $\bar g_{t+1}$ have the same probability distribution.

Furthermore, by the linearity of the expectation, one can obtain (ii):
\begin{align*}
     \sqrt{1- \beta_{t+1}}E(\bar g_t) + E(g_{t+1}) &=  E(\sqrt{1- \beta_{t+1}} \bar g_t + g_{t+1}) \\
    &= E(\bar g_{t+1})
\end{align*}

Thus, we have:
\begin{equation*}
    x_{t+1} = \sqrt{\bar \alpha_{t+1}}x_0 + (\bar g_{t+1} - E(\bar g_{t+1}))
\end{equation*}
which ends the proof by induction.
\end{proof}
Similarly to Eq.\ref{eq:DDPM_update} by using Langevin dynamics, the inference is given by:
\begin{equation}\label{eq:sample_gamma}
x_{t-1} = \dfrac{x_t - \frac{1 - \alpha_t}{\sqrt{1- \bar\alpha_t}}\varepsilon_\theta(x_t, t)}{\sqrt{\bar\alpha_t}} + \sigma_t \dfrac{\bar g_t - E(\bar g_t)}{\sqrt{V(\bar g_t)}}
\end{equation}

In Algorithm~\ref{alg:single_gamma_train} we describe the training procedure. As input we have the: (i) initial scale $\theta_0$, (ii) the dataset $d$, (iii) the number of maximum step in the diffusion process $T$ and (iv) the noise schedule $\beta_1,...,\beta_T$. The training algorithm sample: (i) an example $x_0$, (ii) number of step $t$ and (iii) noise $\varepsilon$. Then it calculates $x_t$ from $x_0$ by using Eq.\ref{eq:close_form_single_gamma}. The neural network $\varepsilon_{\theta}$ has an input $x_t$ and is condition on the time step $t$.
Then it takes a gradient descend step to approximate the normalized noise $\frac{\bar g_t - \bar k_t \theta_t}{\sqrt{1 - | \bar \alpha_t|}}$ with the neural network $\varepsilon_{\theta}$. The main changes between Algorithm~\ref{alg:single_gamma_train} and the single Gaussian case (i.e. Alg.~\ref{alg:DDPM_train}) is the following: (i) calculating the Gamma parameters, (ii) $x_t$ update equation and (iii) the gradient update equation. 

The inference procedure is given in Algorithm~\ref{alg:single_gamma_infernce}. The algorithm starts from a zero mean noise $x_T$ sampled from $\Gamma(\theta_T, \bar k_T)$. Then for $T$ steps the algorithm estimates $x_{t-1}$ from $x_t$ by using Eq.\ref{eq:sample_gamma}. Note that as in \cite{song_denoising_2020} $\sigma_t=\beta_t$. Algorithm~\ref{alg:single_gamma_infernce} change the single Gaussian (i.e. Alg.~\ref{alg:DDPM_inf}) with the following: (i) the start sampling point $x_T$, (ii) the sampling noise $z$ and (iii) the $x_t$ update equation.

\begin{figure}[!t]
\begin{minipage}[t]{0.49\textwidth}
\begin{algorithm}[H]
  \caption{Gamma Training Algorithm} 
  \label{alg:single_gamma_train}
  \begin{algorithmic}[1]
  \STATE Input: initial scale $\theta_0$, dataset $d$, diffusion process length $T$, noise schedule $\beta_1,...,\beta_T$
   \REPEAT
   \STATE $x_0 \sim d(x_0)$
   \STATE $t \sim \mathcal{U}(\{1, ..., T\})$
   \STATE $\bar g_t \sim \Gamma(\bar k_t, \theta_t)$
   \STATE $x_t = \sqrt{\bar \alpha_t}x_0 + (\bar g_t -\bar k_t\theta_t)$
   \STATE Take gradient descent step on: \newline $\| \frac{\bar g_t - \bar k_t \theta_t}{\sqrt{1 - | \bar \alpha_t|}} - \varepsilon_{\theta}(x_t, t)\|_1 $
   \UNTIL{converged}
  \end{algorithmic}
\end{algorithm}
\end{minipage}
\hfill
\begin{minipage}[t]{0.49\textwidth}
\begin{algorithm}[H]
  \caption{Gamma Inference Algorithm}%
  \label{alg:single_gamma_infernce}
  \begin{algorithmic}[1]
   \STATE $ \gamma \sim \Gamma(\theta_T, \bar k_T) $
   \STATE $x_T = \gamma - \theta_T*\bar k_T$
   \FOR{t = T, ..., 1}
      \STATE $x_{t-1} = \frac{x_t -\frac{1-\alpha_t}{\sqrt{1 - \bar\alpha_t}} \epsilon(x_t, t)}{\sqrt{\alpha_t}}$
    \IF{t > 1}
        \STATE $z \sim \Gamma(\theta_{t-1}, \bar k_{t-1})$
        \STATE $z = \frac{z - \theta_{t-1}\bar k_{t-1}}{\sqrt{(1- \bar \alpha_t)}}$
        \STATE $x_{t-1} = x_{t-1} + \sigma_t z$
    \ENDIF
    \ENDFOR
  \end{algorithmic}
\end{algorithm}
\end{minipage}
\end{figure}

\section{Experiments} 

\subsection{Speech Generation} 
For our speech experiments we used a version of Wavegrad \cite{chen_wavegrad_2020} based on this implementation \cite{ivangit} (under BSD-3-Clause License). We evaluate our model with high level perceptual quality of speech measurements which are PESQ \cite{PESQ_paper}, STOI \cite{STOI_paper} and MCD \cite{mcd_paper}. We used the standard Wavegrad method with the Gaussian diffusion process as a baseline. We use two Nvidia Volta V100 GPUs to train our models.

For all the experiments, the inference noise schedules ($\beta_0, .., \beta_T$) were defined as described in the Wavegrad paper \cite{chen_wavegrad_2020}. For $1000$ and $100$ iterations the noise schedule is linear, for $25$ iterations it comes from the Fibonacci and for $6$ iteration we performed a grid search that is model-dependent to find the best noise schedule parameters. For other hyper-parameters (e.g. learning rate, batch size, etc) we use the same as appearing in Wavegrad~\cite{chen_wavegrad_2020}.

For Gaussian mixture noise, we set $C=2$ which is a mixture of two Gaussians. To ensure the symmetry of the distribution with respect to $0$ we used equal probability of sampling from both Gaussian (i.e. $p_1 = p_2 = 0.5$). Wavegrad is trained using a continuous noise level (i.e. $\sqrt{\bar \alpha_t}$) instead of a time step embedding conditioning. Therefore, we set the standard deviation $\phi_t$ of the Gaussians to be linearly dependent on the noise level $\phi_t = \phi_{start} + (\phi_{end} - \phi_{start})\sqrt{\bar \alpha_t}$

where $\phi_{end}$ and $\phi_{start}$ is hyper-parameters. We empirically observe that results are best starting from $\phi_{start} = 1$ and going toward $\phi_{end} = 0.5$. With $\phi_{start} = 1$ from Eq.\ref{eq:mean_relation} the added value at the last step is a single Gaussian centred at $0$. Intuitively, at the end of the inference process we should add small amount of noise which conduct with single Gaussian centred at zero. 

For Gamma noise, the training was performed using the following form of Eq.~\ref{eq:gamma_t}, e.g. $\theta_t = \sqrt{\bar \alpha_t} \theta_0$ and $k_t = \dfrac{\beta_t}{\bar\alpha_t{\theta_0}^2}$. Our best result were obtained using $\theta_0 = 0.001$. 

\noindent{\bf Results\quad }
In Tab.~\ref{tab:LJ} we present the PESQ, STOI, and MCD measurement for LJ dataset \cite{ljspeech17} (under Public Domain license). As can be seen, for a mixture of Gaussian our results are better than the Wavgrad baseline for all number of iterations in both PESQ and STOI. For the Gamma noise distribution our results are better than Wavgrad for $6,25,100$ iteration in both scores. For $1000$ our STOI is better, while the PESQ scores of the model based on the Gamma distribution is inferior but very close to the baseline.

In terms of the MCD score, for a mixture of Gaussian our results are better for all iterations. For Gamma distribution noise even though improving speech distance metrics (e.g. STOI and PESQ) over Wavegrad, our results are worse in MCD measurement. We quantitatively observe that using the Gamma distribution, adds an amount of noise in the generated samples which leads to less sharp spectrograms and hurts the MCD measurement. We believe this is due to the bounded nature of the gamma distribution. Sample results can be found under the following link: \href{https://enk100.github.io/Non-Gaussian-Denoising-Diffusion-Models}{https://enk100.github.io/Non-Gaussian-Denoising-Diffusion-Models/}

\begin{table}[]
\centering
\caption{PESQ, STOI, and MCD metrics for the LJ dataset for various Wavgrad-like models.}
\label{tab:LJ}
\begin{tabular}{@{}l@{~~}c@{~~}c@{~~}c@{~~}c@{~~}c@{~~}c@{~~}c@{~~}c@{~~}c@{~~}c@{~~}c@{~~}c@{}}
\toprule
& \multicolumn{4}{c}{PESQ ($\uparrow$)} & \multicolumn{4}{c}{STOI ($\uparrow$) } & \multicolumn{4}{c}{MCD ($\downarrow$)}\\
\cmidrule(lr){2-5}
\cmidrule(lr){6-9}
\cmidrule(lr){10-13}
Model $\setminus$ Iteration & 6    & 25    & 100   & 1000  & 6    & 25    & 100   & 1000  & 6    & 25    & 100   & 1000\\
\midrule 
Single Gaussian (\cite{chen_wavegrad_2020}) & 2.78 & 3.194 & 3.211 & 3.290 & 0.924 & 0.957 & 0.958 & 0.959 & 2.76 & 2.67 & 2.64 & 2.65 \\
Mixture Gaussian    & 2.84 & \textbf{3.267}  & \textbf{3.321}  & \textbf{3.361} &  0.925  & 0.961   & 0.964   & 0.965  &  \textbf{2.69}  & \textbf{2.46} & \textbf{2.44} & \textbf{2.43}  \\
Gamma Distribution    & \textbf{3.07} & 3.208 & 3.214 & 3.208 & \textbf{0.948} & \textbf{0.972}  & \textbf{0.969} & \textbf{0.969} & 2.89 & 2.85 & 2.86 & 2.84 \\
\bottomrule
\end{tabular}
\end{table}

\subsection{Image Generation} 
Our model is based on the DDIM implementation available in \cite{DDIMgithub} (under the MIT license).  We trained our model on two image datasets (i) CelebA 64x64 \cite{liu2015faceattributes} and (ii) LSUN Church 256x256 \cite{yu15lsun}. The Fréchet Inception Distance (FID)~\cite{heusel2017gans} is used as the benchmark metric. For all experiments, similarly to previous work~\cite{song_denoising_2020}, we compute the FID score with $50,000$ generated images using the torch-fidelity implementation~\cite{torchfidelity}. Similarly to \cite{song_denoising_2020}, the training noise schedule $\beta_1, ... ,\beta_T$ is linearly with values raging from $0.0001$ to $0.02$. For other hyper parameters (e.g. learning rate, batch size etc) we use the same parameters that appear in DDPM \cite{ho2020denoising}. We use eight Nvidia Volta V100 GPUs to train our models.

For Image Generation using the Gaussian mixture we set the parameter $\phi_t$ to be linearly dependent on time step number $t$ $\phi_t = \phi_{start} + (\phi_{end} - \phi_{start})\frac{t}{T}$
Empirically we found that results are best with settings similar to those employed in the speech experiment $\phi_{start}=1$ and $\phi_{end}=0.5$. As observed in the speech experiment, the model performs best when the start of the diffusion process has a noise which most likely value is zero. It led us to the use of $\phi_{start}=1$. 
For Gamma distribution we use $\theta_0=0.001$.

\noindent{\bf Results\quad}
{
We test our models with the inference procedure from DDPM \cite{ho2020denoising} and DDIM \cite{song_denoising_2020}. In Tab.~\ref{tab:celeba} we provide the FID score for CelebA (64x64) dataset \cite{liu2015faceattributes} (under non-commercial research purposes license). As can be seen for DDPM inference procedure for $10,20,50$ steps, the best results obtained from the mixture of Gaussian model, which improves the results by a gap of $268$ FID scores for ten iterations. For $100$ iteration, the best model is the Gamma which improves the results by $31$ FID scores. For $1000$ iteration, the best results obtained from the DDPM model, Nevertheless, our Gamma model obtains results that are closer to the DDPM by a gap of $0.83$. For the DDIM procedure, the best results obtained with the Gamma model for all number of iteration. Furthermore, the mixture of Gaussian model improves the results of the DDIM for $10,20,50,100$ and obtain comparable results for $1000$ iterations. {Fig.~\ref{fig:evolution_celebA} presents samples generated by the three models. Our models provide better quality images when comparing to single Gaussian method.}

In Tab.~\ref{tab:church} we provide the FID score for the LSUN church dataset \cite{yu15lsun} (under unknown license). As can be seen, the Gamma model improves the results over the baseline for $10,20,50,100$ iterations. 

We did not obtain the results for the mixture of Gaussian model due to costly training of the diffusion process to such high resolution dataset.}

\begin{table}[]
\centering
\caption{FID score comparison for CelebA(64x64) dataset. Lower is better.}
\label{tab:celeba}
\begin{tabular}{lccccc}
\toprule
Model $\setminus$ Iteration & 10     & 20     & 50    & 100   & 1000      \\
\midrule 
Single Gaussian DDPM \cite{ho2020denoising}   & 299.71 & 183.83 & 71.71 & 45.2  & \textbf{3.26}    \\
Mixture Gaussian DDPM (ours)   & \textbf{31.21}  & \textbf{25.51}  & \textbf{18.87} & 14.69 & 5.57  \\
Gamma Distribution DDPM (ours)   & 35.59  & 28.24  & 20.24 & \textbf{14.22} & 4.09   \\
\midrule 
Single Gaussian DDIM  \cite{song_denoising_2020}   & 17.33  & 13.73  & 9.17  & 6.53  & 3.51      \\
Mixture Gaussian DDIM  (ours)   & 12.01  & 9.27   & 7.32  & 6.13  & 3.71     \\
Gamma Distribution DDIM (ours)    & \textbf{11.64}  & \textbf{6.83}   & \textbf{4.28}   & \textbf{3.17}  & \textbf{2.92}      \\
\bottomrule
\end{tabular}
\end{table}

\begin{table}[]
\centering
\caption{FID score comparison for LSUN Church (256x256) dataset. Lower is better.}
\label{tab:church}
\begin{tabular}{lcccc}
\toprule
Model $\setminus$ Iteration & 10     & 20     & 50    & 100   \\ 
\midrule 
Single Gaussian DDPM \cite{ho2020denoising}     &  51.56 &	23.37 &	11.16 &	8.27 \\ 
Gamma Distribution DDPM (ours)      &  \textbf{28.56} &	\textbf{19.68} &	\textbf{10.53} &	\textbf{7.87} \\ 
\midrule 
Single Gaussian DDIM  \cite{song_denoising_2020}     &  19.45 &	12.47 &	10.84	& 10.58 \\ 
Gamma Distribution DDIM (ours)    &  \textbf{18.11} &	\textbf{11.32} &	\textbf{10.31} &	\textbf{8.75}  \\ 
\bottomrule
\end{tabular}
\end{table}

\begin{figure}
\centering
\begin{tabular}{c}
\includegraphics[width=\linewidth]{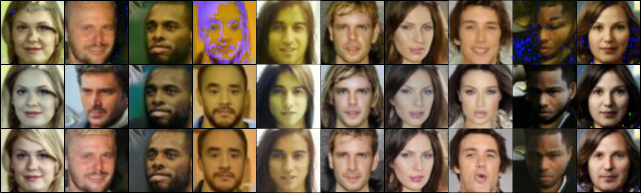}\\ 
\end{tabular}
\smallskip
\caption{Typical examples of images generated with $100$ iterations and $\eta=0$. For three different models trained with different noise distributions - (i) First row - single Gaussian noise  (ii) Second row - Mixture of Gaussian noise, (iii) Third row - Gamma noise. All models start from the same noise instance.}
\label{fig:evolution_celebA}
\end{figure}

\section{Limitations}
While we were able to add two new distributions to the toolbox of diffusion methods, we were not able to provide a complete characterization of all suitable distributions. This effort is left as future work. Additionally, our work suggests that the Gaussian noise should be replaced. However, we still need to identify the conditions in which each of the distributions would outperform the others. 

\section{Conclusions} 
{
We present two novel diffusion models. The first employs a mixture of two Gaussians and the second the  Gamma noise distribution.  A key enabler for using these distributions is a closed form formulation (Eq.~\ref{eq:closed_form_gauss_mix} and Eq.~\ref{eq:close_form_single_gamma}) of the multi-step noising process, which allows for efficient training. These two models improve the quality of the generated image and audio as well as the speed of generation in comparison to conventional Gaussian-based diffusion processes.} 

\section{Acknowledgments}
 This project has received funding from the European Research Council (ERC) under the European Unions Horizon 2020 research and innovation programme (grant ERC CoG 725974). The contribution of Eliya Nachmani is part of a Ph.D. thesis research conducted at Tel Aviv University.

\bibliographystyle{plain}
\bibliography{neurips_2021}

\end{document}